\documentclass[11pt]{article}
\usepackage[in]{fullpage}
\usepackage[parfill]{parskip}
\usepackage{amsmath,amsthm,amssymb,bbm}
\usepackage{mathtools}
\usepackage{cases}
\usepackage{dsfont}
\usepackage{microtype}
\usepackage{subfigure}
\usepackage{algorithm,algorithmic}
\usepackage{color}
\usepackage{appendix}

\usepackage{bm}
\usepackage{subfigure}
\usepackage{algorithm,algorithmic}
\usepackage{color}
\usepackage{booktabs}       
\usepackage{appendix}
\usepackage{authblk}

\usepackage{url}
\usepackage[authoryear]{natbib}
\usepackage[colorlinks,citecolor=blue,urlcolor=blue,linkcolor=blue,linktocpage=true]{hyperref}
\usepackage{cleveref} 
\usepackage{enumitem}

\newtheorem{theorem}{Theorem}[section]
\newtheorem{lemma}[theorem]{Lemma}

\newtheorem{corollary}[theorem]{Corollary}
\newtheorem{definition}[theorem]{Definition}

\newtheorem{assumption}[theorem]{Assumption}

\def\E{\mathbb{E}}

\begin{document}

\title{Logarithmic Switching Cost in Reinforcement Learning\\ beyond Linear MDPs}
\author[1]{Dan Qiao}
\author[1,2]{Ming Yin}
\author[1]{Yu-Xiang Wang}
\affil[1]{Department of Computer Science, UC Santa Barbara}
\affil[2]{Department of Statistics and Applied Probability, UC Santa Barbara}
\affil[ ]{\texttt{\{danqiao,ming\_yin\}@ucsb.edu}, \;
	\texttt{yuxiangw@cs.ucsb.edu}}

\date{}

\maketitle

\begin{abstract}
 In many real-life reinforcement learning (RL) problems, deploying new policies is costly. In those scenarios, algorithms must solve exploration (which requires adaptivity) while switching the deployed policy sparsely (which limits adaptivity). In this paper, we go beyond the existing state-of-the-art on this problem that focused on linear Markov Decision Processes (MDPs) by considering linear Bellman-complete MDPs with low inherent Bellman error. We propose the ELEANOR-LowSwitching algorithm that achieves the near-optimal regret with a switching cost logarithmic in the number of episodes and linear in the time-horizon $H$ and feature dimension $d$. We also prove a lower bound proportional to $dH$ among all algorithms with sublinear regret. In addition, we show the ``doubling trick'' used in ELEANOR-LowSwitching can be further leveraged for the generalized linear function approximation, under which we design a sample-efficient algorithm with near-optimal switching cost.
\end{abstract}

\newpage
\tableofcontents
\newpage

\section{Introduction}\label{sec:intro}
In many real-world reinforcement learning (RL) tasks, limited computing resources make it challenging to apply fully adaptive algorithms that continually update the exploration policy. As a surrogate, it is more cost-effective to collect data in large batches using the current policy and make changes to the policy after the entire batch is completed. For example, in a recommendation system \citep{afsar2021reinforcement}, it is easier to gather new data quickly, but deploying a new policy takes longer as it requires significant computing and human resources. Therefore, it's not feasible to switch policies based on real-time data, as typical RL algorithms would require. A practical solution is to run several experiments in parallel and make decisions on policy updates only after the entire batch has been completed. Similar limitations occur in other RL based applications such as healthcare \citep{yu2021reinforcement}, robotics \citep{kober2013reinforcement}, and new material design \citep{zhou2019optimization}, where the agent must minimize the number of policy updates while still learning an effective policy using a similar number of trajectories as fully-adaptive methods. On the theoretical side, \citet{bai2019provably} brought up the definition of \emph{switching cost}, which measures the number of policy updates. In this paper, we measure the adaptivity of online reinforcement learning algorithms via \emph{global switching cost}, and we leave the formal definition to Section \ref{sec:setup}.

In recent years, there has been a growing interest in designing online reinforcement learning algorithms with low switching costs \citep{bai2019provably,zhang2020almost,qiao2022sample,gao2021provably,wang2021provably,kong2021online,velegkas2022best}. While much progress has been made in achieving near-optimal results, most of the research has focused on the tabular MDP setting and the slightly more general linear MDP setting \citep{yang2019sample,jin2020provably}. However, linear MDP is still a restrictive model, and subsequent works have proposed a variety of more general settings, such as low inherent Bellman error \citep{zanette2020learning}, generalized linear function approximation \citep{wang2019optimism}, low Bellman rank \citep{jiang2017contextual}, low rank \citep{agarwal2020flambe}, and low Bellman eluder dimension \citep{jin2021bellman}. Therefore, it is natural to question whether reinforcement learning with low switching cost is achievable under these more general MDP settings.

\begin{table}\label{tab}
	\centering
	{
		\begin{tabular}{ |c|c|c|c| } 
			\hline
			\textit{Algorithms for regret minimization} & \textit{Setting} & \textit{Regret bound}  &  \textit{Switching cost bound} \\
			\hline 
			
			\textcolor{blue}{Our Algorithm \ref{alg:main} (Theorem \ref{thm:main})$^\dagger$} & \textcolor{blue}{Low IBE} & \textcolor{blue}{$\widetilde{O}\left(\sum_{h=1}^H d_h \sqrt{K}\right)$}  & \textcolor{blue}{$O(\sum_{h=1}^H d_h\log K)$} \\
			
			\textcolor{blue}{Our Algorithm \ref{alg:minor} (Theorem \ref{thm:minor})$^\star$} & \textcolor{blue}{GLM} & \textcolor{blue}{$\widetilde{O}\left(H\sqrt{d^3K}\right)$}  & \textcolor{blue}{$O(dH\log K)$} \\
			
			Algorithm 1 of \citet{gao2021provably}$^\ddag$ & Linear MDP & $\widetilde{O}(\sqrt{d^3H^4K})$ & $O(dH\log K)$ \\
			
			UCB-Advantage \citep{zhang2020almost} & Tabular MDP & $\widetilde{O}(\sqrt{H^3SAK})$ & $O(H^2SA\log K)^*$ \\
			
			APEVE \citep{qiao2022sample} & Tabular MDP & $\widetilde{O}(\sqrt{H^5S^2AK})$ & $O(HSA\log\log K)$ \\
			
			\hline
			\textcolor{blue}{Lower bound (Theorem \ref{thm:lower})} & \textcolor{blue}{Low IBE} & \textcolor{blue}{If ``no-regret''} & \textcolor{blue}{$\Omega(\sum_{h=1}^H d_h)$}  \\
			\hline
			\textcolor{blue}{Lower bound (Theorem \ref{thm:lower2})} & \textcolor{blue}{GLM} & \textcolor{blue}{If ``no-regret''} &  \textcolor{blue}{$\Omega(dH)$}  \\
			\hline
		\end{tabular}
	}
	\caption{Comparison of our results (in \textcolor{blue}{blue}) to existing works regarding regret bound and (global) switching cost bound. ``Low IBE'' is short for low inherent Bellman error while ``GLM'' represents generalized linear function approximation, where both settings generalize linear MDP. For both ``Low IBE'' and ``GLM'' settings, we assume the total reward is bounded by $1$. In particular, we show the regret bound for ``Low IBE'' assuming the inherent Bellman error is $0$ while the detailed result is shown in Theorem \ref{thm:main}. We highlight that our switching cost upper bounds under both settings match the corresponding lower bounds up to logarithmic factors. $^\dagger$: Here $d_h$ is the dimension of feature map for the $h$-th layer and $K$ is the number of episodes. When applied to linear MDP, there will be an additional factor of $H$ in the regret bound while $d_h=d$ for all $h$. Therefore, regret bound and switching cost bound will be $\widetilde{O}(\sqrt{d^2H^4K})$ and $O(dH\log K)$, respectively. $\star$: When applied to linear MDP, there will be an additional factor of $H$ in the regret bound, and the regret bound will be $\widetilde{O}(\sqrt{d^3H^4K})$. $\ddag$: This result is generalized by \citet{wang2021provably} whose algorithm has a same switching cost bound under this regret bound. $*$: The switching cost here is local switching cost (defined in \citet{bai2019provably}), which is specified to tabular MDP.}
\end{table}

\noindent\textbf{Our contributions.} In this paper, we extend previous results under linear MDP to its two natural extensions, linear Bellman-complete MDPs with low inherent Bellman error \citep{zanette2020learning} and MDP with genaralized linear function approximation \citep{wang2019optimism}. Under both settings, we design algorithms with near optimal regret and switching cost. Our contributions are three-fold and summarized as below. 

\begin{itemize}
	\item A new algorithm (Algorithm \ref{alg:main}) based on ``doubling trick'' for regret minimization under the low inherent Bellman error setting that achieves global switching cost of $O(\sum_{h=1}^H d_h\log K)$ and regret of $\widetilde{O}\left(\sum_{h=1}^H d_h \sqrt{K}+\sum_{h=1}^H \sqrt{d_h}\mathcal{I}K\right)$, where $d_h$ is the dimension of feature map for the $h$-th layer, $\mathcal{I}$ is the inherent Bellman error and $K$ is the number of episodes (Theorem \ref{thm:main}). The regret bound is known to be minimax optimal \citep{zanette2020learning}.
	
	\item When the inherent Bellman error $\mathcal{I}=0$, we prove a nearly matching switching cost lower bound (Theorem \ref{thm:lower}) $\Omega(\sum_{h=1}^H d_h)$ for any algorithm with sub-linear regret bound, which implies that the switching cost of our Algorithm \ref{alg:main} is optimal up to $\log K$ factor. When applied to linear MDP, Algorithm \ref{alg:main} achieves the same switching cost and better regret bound compared to the previous results \citep{gao2021provably,wang2021provably}.
	
	\item We leverage the ``doubling trick'' used in Algorithm \ref{alg:main} under the generalized linear function approximation setting and propose Algorithm \ref{alg:minor} which achieves switching cost of $O(dH\log K)$ and regret of $\widetilde{O}\left(H\sqrt{d^3K}\right)$, where $d$ is the dimension of feature map (Theorem \ref{thm:minor}). We also prove a nearly matching switching cost lower bound of $\Omega(dH)$ for any algorithm with sub-linear regret bound (Theorem \ref{thm:lower2}). The pair of results strictly generalize previous results under linear MDP \citep{gao2021provably,wang2021provably}. 
\end{itemize}

\subsection{Related works}
There is a large and growing body of literature on the statistical theory of reinforcement learning that we will not attempt to thoroughly review. Detailed comparisons with existing work on reinforcement learning with low switching cost \citep{gao2021provably,wang2021provably,zhang2020almost,qiao2022sample} are given in Table \ref{tab}. Notably, the settings we consider are more general than the well studied tabular or linear MDP, while our results for regret and switching cost are comparable or better than the best known results under linear MDP \citep{gao2021provably,wang2021provably}. While there are low adaptive algorithms under other more general settings than linear MDP, they either consider only pure exploration (without regret guarantee) \citep{jiang2017contextual,sun2019model}, or suffer from sub-optimal results comparing to our results \citep{kong2021online,velegkas2022best}.

In addition to switching cost, there are other measurements of adaptivity. The closest measurement is batched learning, which requires decisions about policy updates to be made at only a few (often predefined) checkpoints but does not constrain the number of policy switches. Batched learning has been considered both under bandits \citep{perchet2016batched,gao2019batched} and RL \citep{wang2021provably,qiao2022sample,zhang2022near} while the settings are restricted to tabular MDP or linear MDP. Meanwhile, \citet{matsushima2020deployment} proposed the notion of \emph{deployment efficiency}, which is similar to batched RL with additional requirement that each policy deployment should have similar size. Deployment efficient RL is studied by some following works \citep{huang2021towards,qiao2022near,modi2021model}. However, as pointed out by \citet{qiao2022near}, deployment complexity is not a good measurement of adaptivity when studying regret minimization.

Technically speaking, we directly base on ELEANOR \citep{zanette2020learning} and Algorithm 1 of \citet{wang2019optimism}, which admit fully adaptive structure. We apply ``doubling trick'' when deciding whether to update the exploration policy, in order to achieve low switching cost. In particular, we show that the ``information gain'' used in previous works under linear MDP \citep{gao2021provably,wang2021provably}: the determinant of empirical covariance matrix can be extended to more general MDPs with linear approximation. Therefore, we only update the exploration policy when the ``information gain'' doubles, and the switching cost depends only logarithmically on the number of episodes $K$. 

\section{Problem setup}\label{sec:setup}

\noindent\textbf{Notations.} Throughout the paper, for $n\in\mathbb{Z}^{+}$, $[n]=\{1,2,\cdots,n\}$. We denote $\|x\|_{\Lambda}=\sqrt{x^\top \Lambda x}$.  For matrix $X\in\mathbb{R}^{d\times d}$, $\|\cdot\|_2$, $\det(\cdot)$, $\lambda_{\min}(\cdot)$, $\lambda_{\max}(\cdot)$ denote the operator norm, determinant, smallest eigenvalue and largest eigenvalue, respectively. In addition, we use standard notations such as $O$ and $\Omega$ to absorb constants while $\widetilde{O}$ and $\widetilde{\Omega}$ suppress logarithmic factors.

\noindent\textbf{Markov Decision Processes.} We consider finite-horizon episodic \emph{Markov Decision Processes} (MDP) with non-stationary transitions, denoted by a tuple $\mathcal{M}=(\mathcal{S}, \mathcal{A}, H, P_h, r_h)$ \citep{sutton1998reinforcement}, where $\mathcal{S}$ is the state space, $\mathcal{A}$ is the action space and $H$ is the horizon. The non-stationary transition kernel has the form $P_h:\mathcal{S}\times\mathcal{A}\times\mathcal{S} \mapsto [0, 1]$  with $P_{h}(s^{\prime}|s,a)$ representing the probability of transition from state $s$, action $a$ to next state $s'$ at time step $h$. In addition, $r_h(s,a)\in\Delta([0,1])$ denotes the corresponding distribution of reward.\footnote{We overload the notation $r$ so that $r$ also denotes the expected (immediate) reward function.} Without loss of generality, we assume there is a fixed initial state $s_{1}$.\footnote{The generalized case where the initial distribution is an arbitrary distribution can be recovered from this setting by adding one layer to the MDP.} A policy can be seen as a series of mapping $\pi=(\pi_1,\cdots,\pi_H)$, where each $\pi_h$ maps each state $s \in \mathcal{S}$ to a probability distribution over actions, \emph{i.e.} $\pi_h: \mathcal{S}\rightarrow \Delta(\mathcal{A})$, $\forall\, h\in[H]$. A random trajectory $ (s_1, a_1, r_1, \cdots, s_H,a_H,r_H,s_{H+1})$ is generated by the following rule: $s_1$ is fixed, $a_h \sim \pi_h(\cdot|s_h), r_h \sim r_h(s_h, a_h), s_{h+1} \sim P_h (\cdot|s_h, a_h), \forall\, h \in [H]$. For normalization, we assume that $\sum_{h=1}^{H} r_h\in[0,1]$ almost surely.

\noindent\textbf{$Q$-values, Bellman operator.} Given a policy $\pi$ and any $h\in[H]$, the value function $V^\pi_h(\cdot)$ and Q-value function $Q^\pi_h(\cdot,\cdot)$ are defined as:
$
V^\pi_h(s)=\mathbb{E}_\pi[\sum_{t=h}^H r_{t}|s_h=s] ,
Q^\pi_h(s,a)=\mathbb{E}_\pi[\sum_{t=h}^H  r_{t}|s_h,a_h=s,a],\;\forall\, s,a\in\mathcal{S}\times\mathcal{A}.
$ Besides, the value function and Q-value function with respect to the optimal policy $\pi^\star$ is denoted by $V^\star_h(\cdot)$ and $Q^\star_h(\cdot,\cdot)$.  
Then the Bellman operator $\mathcal{T}_h$ applied to $Q_{h+1}$ is defined as 
\begin{align*}
\mathcal{T}_h(Q_{h+1})(s,a)=r_h(s,a)+\E_{s^\prime\sim P_h(\cdot|s,a)}\max_{a^\prime}Q_{h+1}(s^\prime,a^\prime).
\end{align*}

\noindent\textbf{Regret.} We measure the performance of online reinforcement learning algorithms by the regret. The regret of an algorithm over $K$ episodes is defined as
$$\text{Regret}(K) := \sum_{k=1}^{K}[V_{1}^\star(s_{1})-V_{1}^{\pi_{k}}(s_{1})],$$
where $\pi_{k}$ is the policy it deploys at episode $k$. Besides, we denote the total number of steps by $T := KH$.

\noindent\textbf{Switching cost.} We adopt the global switching cost \citep{bai2019provably}, which simply measures how many times the algorithm changes its policy:
$$N_{switch} := \sum_{k=1}^{K-1} \mathds{1}\{\pi_{k}\neq \pi_{k+1}\} .$$
Global switching cost is a widely applied measurement of the adaptivity of an online RL algorithm both under the tabular setting \citep{bai2019provably,zhang2020almost,qiao2022sample} and the linear MDP setting \citep{gao2021provably,wang2021provably}. Similar to previous works, our algorithm also uses deterministic policies only.

\subsection{Low inherent Bellman error}\label{sec:setting}
In this part, we introduce the linear function approximation, the definition of inherent Bellman error \citep{zanette2020learning} and the connection between the low inherent Bellman error setting and the linear MDP setting \citep{jin2020provably}.

To encode linear function approximation of the state space $\mathcal{S}$, a common approach is to define a feature map $\phi_h:\mathcal{S}\times\mathcal{A}\rightarrow \mathbb{R}^{d_h}$, which can be different across different timestep. Then the $Q$-value functions are represented as linear functions of $\phi_h$, \emph{i.e.}, $Q_h(s,a)=\phi_h(s,a)^\top \theta_h$ for some $\theta_h\in\mathbb{R}^{d_h}$.

The feasible parameter class for timestep $h$ is defined as 
$$\mathcal{B}_h:=\{\theta_h\in\mathbb{R}^{d_h}|\;|\phi_h(s,a)^\top \theta_h|\leq 1,\;\forall\;(s,a)\},$$
which is consistent with our assumption that $Q^\pi_h(s,a)\leq 1$.

For each feasible parameter $\theta\in\mathcal{B}_h$, the corresponding $Q$-value function and value function are defined as 
$$Q_h(\theta)(s,a)=\phi_h(s,a)^\top \theta,\;\; V_h(\theta)(s)=\max_a \phi_h(s,a)^\top \theta.$$
Meanwhile, the associated function spaces are
$$\mathcal{Q}_h:=\{Q_h(\theta_h)\,|\,\theta_h\in\mathcal{B}_h\},\;\mathcal{V}_h:=\{V_h(\theta_h)\,|\,\theta_h\in\mathcal{B}_h\}.$$

Similar to \citet{zanette2020learning}, we make the following normalization assumption, which is without loss of generality.
$$\|\phi_h(s,a)\|_2\leq1,\;\;\forall\,(h,s,a)\in[H]\times\mathcal{S}\times\mathcal{A}.$$
$$\|\theta_h\|_2\leq\sqrt{d_h},\;\;\forall\,h\in[H],\,\theta_h\in\mathcal{B}_h.$$

\noindent\textbf{Inherent Bellman error.} For provably efficient learning, completeness assumption is widely adopted \citep{zanette2020learning,wang2020reinforcement,jin2021bellman}. In this paper, we characterize the completeness by assuming an upper bound of the projection error when we project $\mathcal{T}_h Q_{h+1}$ ($Q_{h+1}\in\mathcal{Q}_{h+1}$) to $\mathcal{Q}_h$. Formally, we have the following definition of inherent Bellman error.

\begin{definition}
	The inherent Bellman error of an MDP with a known linear feature map $\{\phi_h(\cdot,\cdot)\}_{h\in[H]}$ is defined as the maximum over the timesteps $h\in[H]$ of
	\begin{equation*}
	\sup_{\theta_{h+1}\in\mathcal{B}_{h+1}}\inf_{\theta_h\in\mathcal{B}_h}\sup_{s,a}|\phi_h(s,a)^\top \theta_h-(\mathcal{T}_h Q_{h+1}(\theta_{h+1}))(s,a)|.
	\end{equation*}
\end{definition}

Similar to \citet{zanette2020learning}, we assume the inherent Bellman error of the MDP is upper bounded by some (known) constant $\mathcal{I}\geq 0$. Below we will show that this setting strictly generalizes the linear MDP setting \citep{jin2020provably}.

\noindent\textbf{Connections to linear MDP.} Since linear MDP admits transition kernel and reward function that is linear in a known feature map $\phi$, for any function $V(\cdot):\mathcal{S}\rightarrow \mathbb{R}$, $\mathcal{T}_h V(\cdot,\cdot)$ is a linear function of $\phi(\cdot,\cdot)$ \citep{jin2020provably}. Therefore, a linear MDP with feature map $\phi$ and dimension $d$ is a special case of the low inherent Bellman error setting with $\mathcal{I}=0$, $\phi_1=\cdots=\phi_H=\phi$ and $d_1=\cdots=d_H=d$ (if ignoring the scale of rewards). More importantly, it is shown that an MDP with zero inherent Bellman error ($\mathcal{I}=0$) may not be a linear MDP \citep{zanette2020learning}, which means that the setting in this paper is strictly more general and technically demanding than linear MDP. For more discussions about the low inherent Bellman error setting and relavent comparisons, please refer to Section 3 in \citet{zanette2020learning}.

\section{Main algorithm}
In this section, we propose our main algorithm: ELEANOR-LowSwitching (Algorithm \ref{alg:main}) and the low switching design for global optimism-based algorithms.  

We begin with the standard LSVI technique. At the beginning of the $k$-th episode, assume the parameter for the $(h+1)$-th layer is fixed to be $\theta_{h+1}$. Then LSVI minimizes the following objective function with respect to $\theta$:

\begin{equation}\label{equ:lsvi}
\sum_{\tau=1}^{k-1}\left((\phi_h^\tau)^\top\theta-r_h^\tau-V_{h+1}(\theta_{h+1})(s_{h+1}^\tau)\right)^2+\lambda\|\theta\|_2^2,
\end{equation}

where $\phi_h^\tau$ is short for $\phi_h(s_h^\tau,a_h^\tau)$ and $r_h^\tau$ is the reward encountered at layer $h$ of the $\tau$-th episode. The minimization problem \eqref{equ:lsvi} has a closed form solution:

\begin{equation}
\widehat{\theta}_h=(\Sigma_h^k)^{-1}\sum_{\tau=1}^{k-1}\phi_h^\tau\left[r_h^\tau+V_{h+1}(\theta_{h+1})(s_{h+1}^\tau)\right],
\end{equation}
where $\Sigma_h^k=\sum_{\tau=1}^{k-1}\phi_h^\tau(\phi_h^\tau)^\top+\lambda I_{d_h}$ is the empirical covariance matrix. 

Based on the standard LSVI, we introduce the global optimistic planning below, where an optimization problem is solved to derive the most optimistic estimate of the Q-value function at the initial state. At each episode where the policy is updated, Algorithm \ref{alg:main} solves the following problem. 

\begin{definition}[Optimistic planning]\label{def:optimization}
	\begin{equation*}
	\begin{split}
	&\max_{\{\bar{\xi}_h\}_{h\in[H]},\{\widehat{\theta}_h\}_{h\in[H]},\{\bar{\theta}_h\}_{h\in[H]}} \max_{a}\phi_1(s_1,a)^\top \bar{\theta}_1\;\;\;\text{subject to}\\
	&\widehat{\theta}_h=(\Sigma_h^k)^{-1}\sum_{\tau=1}^{k-1}\phi_h^\tau\left(r_h^\tau+V_{h+1}(\bar{\theta}_{h+1})(s_{h+1}^\tau)\right),\\
	&\bar{\theta}_h=\widehat{\theta}_h+\bar{\xi}_h;\;\;\;\|\bar{\xi}_h\|_{\Sigma_h^k}\leq \sqrt{\alpha_h^k};\;\;\;\bar{\theta}_h\in\mathcal{B}_h.
	\end{split}
	\end{equation*}
\end{definition}

Definition \ref{def:optimization} optimizes over the perturbation $\bar{\xi}_h$ added to the least square solution $\widehat{\theta}_h$. The constraint on $\bar{\xi}_h$ is
\begin{equation}
\|\bar{\xi}_h\|_{\Sigma_h^k}\leq \sqrt{\alpha_h^k} := \widetilde{O}(\sqrt{d_h+d_{h+1}})+\sqrt{k}\mathcal{I},
\end{equation}
where the definition of $\sqrt{\alpha_h^k}$ will be specified in Appendix \ref{sec:pr}. As will be shown in the analysis, the first term accounts for the estimation error of the LSVI, while the second term accounts for the model misspecification (recall that $\mathcal{I}$ is inherent Bellman error). Finally, with high probability, there will be a valid solution of the optimization problem (details in Appendix \ref{sec:pr}), and therefore Algorithm \ref{alg:main} is well posed.

\noindent\textbf{About global optimism.} We highlight that the optimization problem aims at being optimistic only at the initial state instead of choosing a value function everywhere optimistic, as in LSVI-UCB \citep{jin2020provably}. Such global optimism effectively keeps the linear structure of our function class and reduces the dimension of the covering set, since we do not need to cover the quadratic bonus as in \citet{jin2020provably}. 

\noindent\textbf{Algorithmic design.} We present the whole learning process
in Algorithm \ref{alg:main}. For linear function approximation, we characterize the ``information gain'' (the information we learned from interacting with the MDP) through the determinant of the empirical covariance matrix $\Sigma_h^k$ (line 5). To achieve low switching cost, we only update the exploration policy when the ``information gain'' doubles for some layer $h\in[H]$ (line 7), and each update means the information about some layer has doubled. As will be shown later, such ``doubling schedule'' will lead to a switching cost depending only logarithmically on $K$, in stark contrast to its fully adaptive counterpart: ELEANOR \citep{zanette2020learning}. When an update occurs, Algorithm \ref{alg:main} solves the optimization problem to derive $\{\bar{\theta}_h\}_{h\in[H]}$ ensuring global optimism (line 9), takes the greedy policy with respect to $\phi_h(\cdot,\cdot)^\top \bar{\theta}_h^k$ (line 10) and updates the empirical covariance matrix (line 11).  

\begin{algorithm}[tbh]
	\caption{ELEANOR-LowSwitching}\label{alg:main}
	\begin{algorithmic}[1]
		\STATE \textbf{Input}: Number of episodes $K$, regularization $\lambda=1$, feature map $\{\phi_h(\cdot,\cdot)\}_{h\in[H]}$, failure probability $\delta$, initial state $s_1$ and inherent Bellman error $\mathcal{I}$.
		\STATE \textbf{Initialize}: $\Sigma_h=\Sigma_h^0=\lambda I_{d_h}$, for all $h\in[H]$.
		\FOR{$k=1,2,\cdots,K$}  
		\FOR{$h=1,2,\cdots,H$}
		\STATE $\Sigma_h^k=\sum_{\tau=1}^{k-1}\phi_h(s_h^\tau,a_h^\tau)\phi_h(s_h^\tau,a_h^\tau)^\top+\lambda I_{d_h}.$
		\ENDFOR
		\IF{$\exists\,h\in[H]$, $\det(\Sigma_h^k)\geq 2\det(\Sigma_h)$}
		\STATE Set $\bar{\theta}_{H+1}^k=\widehat{\theta}_{H+1}^k=\bar{\xi}_{H+1}^k=0$.
		\STATE Solve the optimization problem in Definition \ref{def:optimization}.
		\STATE Set $\pi_h^k(s)=\arg\max_{a}\phi_h(s,a)^\top \bar{\theta}_h^k$, $\forall\,h\in[H]$.
		\STATE Set $\Sigma_h=\Sigma_h^k$, $\forall\,h\in[H]$.
		\ELSE 
		\STATE Set $\pi_h^k=\pi_h^{k-1}$ for all $h\in[H]$.
		\ENDIF
		\STATE Deploy policy $\pi_k=(\pi^k_1,\cdots,\pi^k_H)$ and get trajectory $(s_1^k,a_1^k,r_1^k,\cdots,s_{H+1}^k)$.
		\ENDFOR
	\end{algorithmic}
\end{algorithm}

\noindent\textbf{Generalization over previous algorithms.} If we remove the update rule in Algorithm \ref{alg:main} and solve Definition \ref{def:optimization} at all episodes, our Algorithm \ref{alg:main} will degenerate to ELEANOR \citep{zanette2020learning}. Compared to ELEANOR, our Algorithm \ref{alg:main} achieves the same regret bound (shown later) and near optimal switching cost. Meanwhile, Algorithm \ref{alg:main} also strictly generalizes the RARELY SWITCHING OFUL algorithm \citep{abbasi2011improved} designed for linear bandits. Taking $H=1$, both our Algorithm \ref{alg:main} and our guarantees (for regret and switching cost) strictly subsumes the RARELY SWITCHING OFUL. In conclusion, we show that low switching cost is possible for RL algorithms with global optimism.

\noindent\textbf{Computational efficiency.} Although Algorithm \ref{alg:main} is shown to be near optimal both in regret and switching cost, the implementation of the optimization problem is inefficient in general. This is because the max operator breaks the quadratic structure of the constraints. 
Such issue also exists for our fully adaptive counterpart: ELEANOR \citep{zanette2020learning}, and other algorithms based on global optimism \citep{jiang2017contextual,sun2019model,jin2021bellman}. We leave the improvement of computation as future work.

\section{Main results}
In this section, we present our main results. We begin with the upper bounds for regret and switching cost. Recall that we assume $\sum_{h=1}^{H} r_h\in[0,1]$ almost surely, while $d_h$ represents the dimension of the feature map for the $h$-th layer and $\mathcal{I}$ is inherent Bellman error.

\begin{theorem}[Main theorem]\label{thm:main}
	The global switching cost of Algorithm \ref{alg:main} is bounded by $O(\sum_{h=1}^H d_h\cdot \log K)$. In addition, with probability $1-\delta$, the regret of Algorithm \ref{alg:main} over $K$ episodes is bounded by
	$$\text{Regret}(K)\leq \widetilde{O}\left(\sum_{h=1}^H d_h \sqrt{K}+\sum_{h=1}^H \sqrt{d_h}\mathcal{I}K\right).$$
\end{theorem}

The proof of Theorem \ref{thm:main} is sketched in Section \ref{sec:upper} with details in the Appendix, below we discuss several interesting aspects of Theorem \ref{thm:main}.

\noindent\textbf{Near-optimal switching cost.} Our algorithm achieves a switching cost that depends logarithmically on $K$, which improves the $O(K)$ switching cost of ELEANOR \citep{zanette2020learning}. We also prove the following information-theoretic limit which says that the switching cost of Algorithm \ref{alg:main} is optimal up to logarithmic factors. Since it is impossible to get sub-linear regret bound with positive inherent Bellman error, we only consider the case where $\mathcal{I}=0$.

\begin{theorem}[Lower bound for no-regret learning]\label{thm:lower}
	Assume that the inherent Bellman error $\mathcal{I}=0$ and $d_h\geq 3$ for all $h\in[H]$, for any algorithm with sub-linear regret bound, the global switching cost is at least $\Omega(\sum_{h=1}^H d_h)$.
\end{theorem}

The proof of Theorem \ref{thm:lower} is sketched in Section \ref{sec:lower} with details in the Appendix.

\noindent\textbf{Application to linear MDP.} As discussed in Section \ref{sec:setting}, linear MDP with dimension $d$ is a special case of the low inherent Bellman error setting with $\mathcal{I}=0$, $d_1=d_2=\cdots=d_H=d$. Therefore, when applied to linear MDP, our Algorithm \ref{alg:main} will have switching cost bounded by $O(dH\log K)$ and regret bounded by $\widetilde{O}(\sqrt{d^2H^3 T})$, where $T=KH$.\footnote{When transferring Theorem \ref{thm:main} to linear MDP, we need to rescale the reward function by $H$, and therefore there will be an additional factor of $H$ in our regret bound.} Compared to current algorithms achieving low switching cost under linear MDP \citep{gao2021provably,wang2021provably}, we achieve the same switching cost and a regret bound better by a factor of $\sqrt{d}$. The improvement on regret bound results from global optimism and a smaller linear function class. More importantly, low inherent Bellman error setting is indeed a harder setting than linear MDP. According to Theorem 2 in \citet{zanette2020learning}, the regret of our Algorithm \ref{alg:main} is minimax optimal. Together with the lower bound of switching cost (Theorem \ref{thm:lower}), Theorem \ref{thm:main} is generally not improvable both in regret and global switching cost. 

\noindent\textbf{Application to misspecified linear bandits.} Taking $H=1$, an MDP with low inherent Bellman error will become a linear bandit \citep{lattimore2020bandit} with model misspecification. For simplicity, we only consider the case where there is no misspecification (\emph{i.e.} $\mathcal{I}=0$), as studied in \citet{abbasi2011improved}. Our result is summarized in the following corollary.

\begin{corollary}[Results under linear bandit]\label{col}
	Suppose $H=1$ and $\mathcal{I}=0$, then the MDP reduces to a linear bandit with dimension $d$. Our Algorithm \ref{alg:main} will reduce to the RARELY SWITCHING OFUL algorithm (Figure 3 in \citet{abbasi2011improved}) and is computationally efficient.
	The global switching cost of Algorithm \ref{alg:main} is $O(d\log K)$, while the regret can be bounded by $\widetilde{O}(d\sqrt{K})$ with high probability.
\end{corollary}

The above corollary is derived by directly plugging $H=1$ and $d_1=d$ in Theorem \ref{thm:main}. Note that our Corollary \ref{col} matches the results in \citet{abbasi2011improved}, and our Algorithm \ref{alg:main} can be applied under the more general case with model misspecification. Therefore, our results can be seen as strict generalization of \citet{abbasi2011improved}.

\section{Proof sketch}\label{sec:proofsketch}

Due to the space constraint, we sketch the proof in this section while more details are deferred to the Appendix. We begin with the proof overview of Theorem \ref{thm:main}.

\subsection{Upper bounds}\label{sec:upper}

\noindent\textbf{Upper bound of switching cost.} Let $\{k_1, k_2, \cdots, k_N\}$ be the episodes where the algorithm updates the policy (N is the global switching cost), and we also define $k_0=0$.

According to the update rule (line 7 of Algorithm \ref{alg:main}), every time the policy is updated, at least one $\det(\Sigma_h^k)$ doubles, which implies that $\Pi_{h=1}^H \det(\Sigma_{h}^{k_{i+1}})\geq 2\Pi_{h=1}^H \det(\Sigma_{h}^{k_{i}})$ for all $i\in[N]$. This further implies $$\Pi_{h=1}^H \det(\Sigma_{h}^{k_{N}})\geq 2^N\Pi_{h=1}^H \det(\Sigma_{h}^{k_{0}}).$$ Since the left hand side can be upper bounded by $K^{\sum_{h=1}^H d_h}$ (details in Lemma D.2) and the right hand side is just $2^N$ (from definition), the global switching cost (\emph{i.e.} $N$) is bounded by $O(\sum_{h=1}^H d_h\log K)$.

Below we give a proof overview of the regret bound.

\noindent\textbf{Upper bound of regret.} We denote $\bar{Q}_{h}^k(\cdot,\cdot)=Q_h(\bar{\theta}_h^k)(\cdot,\cdot)=\phi_h(\cdot,\cdot)^\top\bar{\theta}_h^k$, where $\bar{\theta}_h^k$ is the solution of Definition \ref{def:optimization} at the $k$-th episode. Similarly, $\bar{V}_h^k(\cdot)=V_h(\bar{\theta}_h^k)(\cdot)$. In addition, let $b_k$ denote the last policy update before episode $k$, for all $k\in[K]$.

Based on concentration inequalities of self-normalized processes, we can show that with high probability, the ``best feasible'' approximant parameter $\theta^\star$ (Definition A.3) is a feasible solution of Definition \ref{def:optimization}. Therefore, the $\bar{V}_1^k(s_1)$ is always a nearly optimistic estimate of $V_1^\star(s_1)$ (summarized in Lemma A.5) and we only need to bound 
\begin{equation}
\begin{split}
&\text{Regret}(K) = \sum_{k=1}^{K}\left(V_{1}^\star(s_{1})-V_{1}^{\pi_{b_k}}(s_{1})\right)\\\leq&HK\mathcal{I}+\sum_{k=1}^{K}\left(\bar{V}_{1}^{b_{k}}(s_1)-V_{1}^{\pi_{b_k}}(s_{1})\right).
\end{split}
\end{equation}

Meanwhile, the pointwise Bellman error can be bounded as (this result is stated in Lemma A.6)
$$\left|\left(\bar{Q}_h^{b_k}-\mathcal{T}_h \bar{Q}_{h+1}^{b_k}\right)(s,a)\right|\leq \mathcal{I}+2\left\|\phi_h(s,a)\right\|_{(\Sigma_h^{b_k})^{-1}}\sqrt{\alpha_h^{b_k}},$$
where $\sqrt{\alpha_h^k}\leq\sqrt{K}\mathcal{I}+\widetilde{O}(\sqrt{d_h+d_{h+1}})$.

As a result, applying regret decomposition accross different layers $h\in[H]$ and bounding the martingale difference by Azuma-Hoeffding inequality (Lemma D.1), we have
\begin{equation}\label{equ:d}
\begin{split}
&\sum_{k=1}^{K}\left(\bar{V}_{1}^{b_{k}}(s_1)-V_{1}^{\pi_{b_k}}(s_{1})\right)\\\leq&\sum_{k=1}^{K}\sum_{h=1}^{H}\left(\mathcal{I}+2\left\|\phi_h(s_h^k,a_h^k)\right\|_{(\Sigma_h^{b_k})^{-1}}\sqrt{\alpha_h^{b_k}}\right)\\&+\text{Sum of bounded martingale difference}\\\leq&
\underbrace{\sum_{k=1}^{K}\sum_{h=1}^{H}2\left\|\phi_h(s_h^k,a_h^k)\right\|_{(\Sigma_h^{b_k})^{-1}}\sqrt{\alpha_h^{b_k}}}_{(a)}\\+&HK\mathcal{I}+\widetilde{O}(\sum_{h=1}^H\sqrt{d_h K}).
\end{split}
\end{equation}

Due to our update rule based on $\det(\Sigma_h^k)$, we have
\begin{equation}\label{equ:a}
\begin{split}
&(a)\leq\sum_{h=1}^H 2\sqrt{\alpha_h^K}\cdot\sqrt{K\sum_{k=1}^K \left\|\phi_h(s_h^k,a_h^k)\right\|_{(\Sigma_h^{b_k})^{-1}}^2}
\\\leq&\sum_{h=1}^H 2\sqrt{\alpha_h^K}\cdot\sqrt{2K\sum_{k=1}^K \left\|\phi_h(s_h^k,a_h^k)\right\|_{(\Sigma_h^{k})^{-1}}^2}
\\\leq&\widetilde{O}\left(\sum_{h=1}^H (\sqrt{K}\mathcal{I}+\sqrt{d_h+d_{h+1}})\cdot \sqrt{Kd_h} \right)\\\leq&
\widetilde{O}(\sum_{h=1}^H \sqrt{d_h}K\mathcal{I}+\sum_{h=1}^H d_h\sqrt{K}),
\end{split}
\end{equation}
where the second inequality holds because of Lemma D.3 and our update rule. The third inequality is from elliptical potential lemma (Lemma D.4).

Finally, the regret bound results from plugging \eqref{equ:a} into \eqref{equ:d}.

\subsection{Lower bound}\label{sec:lower} In this part, we sketch the proof of Theorem \ref{thm:lower}. 

We construct a hard MDP case with zero inherent Bellman error ($\mathcal{I}=0$), which has deterministic transition kernel. Therefore, deploying some deterministic policy will lead to a deterministic trajectory, like pulling an ``arm'' in the multi-armed bandits (MAB) setting. We further show that the number of such ``arms'' is at least $\Omega(\sum_{h=1}^H d_h)$. Together with the lower bounds of switching cost in multi-armed bandits \citep{qiao2022sample}, we can derive the $\Omega(\sum_{h=1}^H d_h)$ lower bound under the low inherent Bellman error setting. 

\section{Extension to generalized linear function approximation}\label{sec:6}

In this section, we consider low adaptive reinforcement learning with generalized linear function approximation \citep{wang2019optimism}. We show that the same ``doubling schedule'' for updating exploration policy (line 7 of Algorithm \ref{alg:main}) can be leveraged under this setting, which enables the design of provably efficient algorithms. We begin with the introduction of generalized linear function approximation.

\subsection{Problem setup}
Different from the low inherent Bellman error setting which characterizes $Q^\star$ using linear functions, we use a function class of generalized linear models (GLMs) to model $Q^\star$. We denote the dimension of feature map by $d$ and define $\mathbb{B}_d=\{x\in\mathbb{R}^d:\|x\|_2\leq 1\}$. 

\begin{definition}[GLM \citep{wang2019optimism}]
	For a known feature map $\phi:\mathcal{S}\times\mathcal{A}\rightarrow\mathbb{B}_d$ and a known link function $f:[-1,1]\rightarrow[-1,1]$, the class of generalized linear models is $\mathcal{G}=\{(s,a)\rightarrow f(\langle\phi(s,a),\theta\rangle):\theta\in\mathbb{B}_d\}$.    
\end{definition}

Similar to \citet{wang2019optimism}, we make the following standard assumption which is without loss of generality.

\begin{assumption}
	$f(\cdot)$ is either monotonically increasing or decreasing. Furthermore, there exist absolute constants $0<\kappa_1<\kappa_2<\infty$ and $M<\infty$ such that $\kappa_1\leq|f^\prime(z)|\leq \kappa_2$ and $|f^{\prime\prime}(z)|\leq M$, for all $|z|\leq 1$.
\end{assumption}

This assumption is naturally satisfied by the identical map $f(z)=z$ and also includes other non-linear maps such as the logistic map $f(z)=1/(1+e^{-z})$.

To characterize completeness under this function class, \citet{wang2019optimism} assumes the function class is closed with respect to the Bellman operator $\mathcal{T}_h$ (defined in Section \ref{sec:setup}). Similarly, we make the same optimistic closure assumption below. Note that for a fixed constant $\Gamma>0$\footnote{$\Gamma$ will be set to depend polynomially on $d$ and $\log K$.}, the enlarged function class is defined as 

\begin{align*}
\mathcal{G}_{\text{up}}=\{&(s,a)\rightarrow\min\{1,f(\langle\phi(s,a),\theta\rangle)+\gamma\|\phi(s,a)\|_A\}:\\
&\theta\in\mathbb{B}_d,\;0\leq \gamma\leq\Gamma,\;A\succcurlyeq 0,\;\|A\|_2\leq 1\}.
\end{align*}

Then the optimistic closure assumption is stated below.
\begin{assumption}
	For all $h\in[H]$ and $g\in\mathcal{G}_{\text{up}}$, we have $\mathcal{T}_h(g)\in\mathcal{G}$.
\end{assumption}

According to Proposition 1 of \citet{wang2019optimism}, this assumption strictly generalizes the standard linear MDP setting by allowing link functions with more expressivity.

\subsection{Low switching algorithm}
We present our Algorithm \ref{alg:minor} below. Intuitively speaking, the algorithmic idea is to apply doubling schedule to Algorithm 1 of \citet{wang2019optimism}. Similar to Algorithm \ref{alg:main}, we only update the exploration policy when the ``information gain'' with respect to some layer has doubled (line 7). When the policy is updated, the LSVI step calculates an estimate of $\theta^\star$ (the parameter w.r.t. the real $Q^\star$ function) iteratively from the $H$-th layer to the first layer through minimizing \eqref{equ:lsvi2}. Then the optimistic $Q$ value function is constructed by adding a bonus term $\gamma\|\phi(\cdot,\cdot)\|_{(\Sigma_h^k)^{-1}}$ to the empirical estimate $f(\phi(\cdot,\cdot)^\top\theta_h^k)$ (line 11). Finally, the greedy policy is deployed for collecting data (line 12, 18). 

\begin{algorithm*}[tbh]
	\caption{Low-switching-cost LSVI-UCB with generalized linear function approximation}\label{alg:minor}
	\begin{algorithmic}[1]
		\STATE \textbf{Input}: Number of episodes $K$, feature map $\{\phi(\cdot,\cdot)\}$, failure probability $\delta$, parameters $\kappa_1,\kappa_2,M$, universal constant $C$.
		\STATE \textbf{Initialize}: $\Sigma_h=\Sigma_h^0=I_{d}$, for all $h\in[H]$. $\gamma=C\kappa_2\kappa_1^{-1}\sqrt{1+M+\kappa_2+d^2\log(\frac{1+\kappa_2+\Gamma}{\delta})}$.
		\FOR{$k=1,2,\cdots,K$}  
		\FOR{$h=1,2,\cdots,H$}
		\STATE $\Sigma_h^k=\sum_{\tau=1}^{k-1}\phi(s_h^\tau,a_h^\tau)\phi(s_h^\tau,a_h^\tau)^\top+I_{d}.$
		\ENDFOR
		\IF{$\exists\,h\in[H]$, $\det(\Sigma_h^k)\geq 2\det(\Sigma_h)$}
		\STATE Set $Q_{H+1}^k(\cdot,\cdot)=0$.
		\FOR{$h=H,\cdots,1$}
		\STATE Solve the empirically optimal estimate of $\theta^\star$. \begin{equation}\label{equ:lsvi2}
		\theta_h^k=\arg\min_{\|\theta\|_2\leq 1}\sum_{\tau=1}^{k-1}(f(\langle\phi(s_h^\tau,a_h^\tau),\theta\rangle)-r_h^\tau-\max_{a^\prime\in\mathcal{A}}Q_{h+1}^k(s_{h+1}^\tau,a^\prime))^2.
		\end{equation}
		\STATE Construct the Q value function: $Q_h^k(\cdot,\cdot)=\min\{1,f(\phi(\cdot,\cdot)^\top \theta_h^k)+\gamma\|\phi(\cdot,\cdot)\|_{(\Sigma_h^k)^{-1}}\}$.
		\STATE Set $\pi_h^k(s)=\arg\max_{a\in\mathcal{A}}Q_h^k(s,a)$.
		\STATE Set $\Sigma_h=\Sigma_h^k$.
		\ENDFOR
		\ELSE 
		\STATE Set $\pi_h^k=\pi_h^{k-1}$ for all $h\in[H]$.
		\ENDIF
		\STATE Deploy policy $\pi_k=(\pi^k_1,\cdots,\pi^k_H)$ and get trajectory $(s_1^k,a_1^k,r_1^k,\cdots,s_{H+1}^k)$.
		\ENDFOR
	\end{algorithmic}
\end{algorithm*}

\subsection{Main results of Algorithm \ref{alg:minor}}
In this part, we state the main results about Algorithm \ref{alg:minor}. We begin with the upper bounds for regret and switching cost. Recall that we still assume $\sum_{h=1}^{H} r_h\in[0,1]$ almost surely, while $d$ represents the dimension of the feature map.

\begin{theorem}[Main results]\label{thm:minor}
	The global switching cost of Algorithm \ref{alg:minor} is bounded by $O(dH\cdot \log K)$. In addition, with probability $1-\delta$, the regret of Algorithm \ref{alg:minor} over $K$ episodes is bounded by
	$$\text{Regret}(K)\leq \widetilde{O}\left(H\sqrt{d^3K}\right).$$
\end{theorem}

The proof of Theorem \ref{thm:minor} is deferred to Appendix \ref{sec:C} due to space limit, below we discuss several interesting aspects of Theorem \ref{thm:minor}.

\noindent\textbf{Near-optimal switching cost.} Our algorithm achieves a switching cost that depends logarithmically on $K$, which improves the $O(K)$ switching cost of Algorithm 1 in \citet{wang2019optimism}. We also prove the following information-theoretic limit which says that the switching cost of Algorithm \ref{alg:minor} is optimal up to logarithmic factors. 

\begin{theorem}[Lower bound for no-regret learning]\label{thm:lower2}
	For any algorithm with sub-linear regret bound, the global switching cost is at least $\Omega(dH)$.
\end{theorem}

Theorem \ref{thm:lower2} is adapted from the lower bound for global switching cost under linear MDP \citep{gao2021provably}, and we leave the proof to Appendix \ref{sec:C}.

\noindent\textbf{Generalization over previous results.} The closest result to our Algorithm \ref{alg:minor} is the fully adaptive Algorithm 1 of \citet{wang2019optimism}, which achieves the same $\widetilde{O}\left(H\sqrt{d^3K}\right)$ regret bound. In comparison, our Algorithm \ref{alg:minor} favors near optimal global switching cost at the same time, which saves computation and accelerates the learning process.

When applying our Algorithm \ref{alg:minor} to the linear MDP case, our Theorem \ref{thm:minor} will imply a regret bound of $\widetilde{O}(\sqrt{d^3H^3T})$\footnote{The identical link function corresponds to $\kappa_1=\kappa_2=1$ and $M=0$. In addition, due to rescaling of reward functions, there will be an additional $H$ factor in the regret bound of Theorem \ref{thm:minor}.} ($T=KH$) and a global switching cost of $O(dH\log K)$, which recovers the results in \citet{gao2021provably,wang2021provably}. Therefore, our result can be considered as generalization of these two results since GLMs allow more general function classes. 

\section{Conclusion and future work}
This paper studied the well motivated problem of online reinforcement learning with low switching cost. Under linear Bellman-complete MDP with low inherent Bellman error, we designed an algorithm (Algorithm \ref{alg:main}) with near optimal regret bound of $\widetilde{O}\left(\sum_{h=1}^H d_h \sqrt{K}+\sum_{h=1}^H \sqrt{d_h}\mathcal{I}K\right)$ and global switching cost bound of $O(\sum_{h=1}^H d_h\cdot \log K)$. In addition, we prove a (nearly) matching global switching cost lower bound $\Omega(\sum_{h=1}^H d_h)$ for any algorithm with sub-linear regret. At the same time, we leverage the same ``doubling trick'' under the generalized linear function approximation setting, and designed a sample-efficient algorithm (Algorithm \ref{alg:minor}) with near optimal switching cost.

Although being more general than linear MDP, the two settings we consider are not the most general ones. The low Bellman eluder dimension setting \citep{jin2021bellman} and MDP with differentiable function approximation \citep{zhang2022off} can be considered as generalization of the two settings in this paper, respectively. Therefore, our results can be considered as a middle step towards low switching reinforcement learning under more general MDP settings. For further extension, it will be interesting to find out whether low switching cost RL is possible under more general MDP settings (\emph{e.g.}, low Bellman eluder dimension \citep{jin2021bellman}, differentiable function class \citep{zhang2022off,yin2023offline}), and we leave these as future work.

\section*{Acknowledgments}
The research is partially supported by NSF Award \#2007117. 

\bibliographystyle{plainnat}
\bibliography{uai2023-template}

\newpage
\appendix

\section{Proof of Theorem \ref{thm:main}}
In this section, we prove our main theorem. We first restate Theorem \ref{thm:main} below, and then prove the bounds for switching cost and regret in Section \ref{sec:ps} and Section \ref{sec:pr}, respectively. 

\begin{theorem}[Restate Theorem \ref{thm:main}]\label{thm:restatemain}
	The global switching cost of Algorithm \ref{alg:main} is bounded by $O(\sum_{h=1}^H d_h\cdot \log K)$. In addition, with probability $1-\delta$, the regret of Algorithm \ref{alg:main} over $K$ episodes is bounded by
	$$\text{Regret}(K)\leq \widetilde{O}\left(\sum_{h=1}^H d_h \sqrt{K}+\sum_{h=1}^H \sqrt{d_h}\mathcal{I}K\right).$$  
\end{theorem}

\subsection{Proof of switching cost bound}\label{sec:ps}
\begin{proof}[Proof of switching cost bound]
	Let $\{k_1, k_2, \cdots, k_N\}$ be the episodes where the algorithm updates the policy, and we also define $k_0=0$. 
	
	According to the update rule (line 7 of Algorithm \ref{alg:main}), for all $i\in[N]$, there exists some $h_i\in[H]$ such that 
	$$\det(\Sigma_{h_i}^{k_{i+1}})\geq 2\det(\Sigma_{h_i}^{k_i}).$$
	In addition, for all $h,i\in[H]\times[N]$, we have 
	$$\det(\Sigma_{h}^{k_{i+1}})\geq \det(\Sigma_{h}^{k_i}).$$
	Combining these two results, we have for all $i\in[N]$,
	\begin{equation}\label{equ:double}
	\Pi_{h=1}^H \det(\Sigma_{h}^{k_{i+1}})\geq 2\Pi_{h=1}^H \det(\Sigma_{h}^{k_{i}}).
	\end{equation}
	Therefore, it holds that
	\begin{equation}\label{equ:switch}
	K^{\sum_{h=1}^H d_h}\geq\Pi_{h=1}^H \det(\Sigma_{h}^{k_{N}})\geq 2^N\Pi_{h=1}^H \det(\Sigma_{h}^{k_{0}})=2^N,
	\end{equation}
	where the first inequality is because of Lemma \ref{lem:wang} and our choice that $\lambda=1$. The second inequality is due to recursive application of \eqref{equ:double}. The last equation holds since we have $\Sigma_h^{k_0}=I_{d_h}$ for all $h$.
	
	Solving \eqref{equ:switch}, we have $N\leq \frac{\sum_{h=1}^H d_h \log K}{\log 2}=O(\sum_{h=1}^H d_h \log K)$, and therefore the proof is complete.
\end{proof}

\subsection{Proof of regret bound}\label{sec:pr}
We first state some technical lemmas from \citet{zanette2020learning}. We begin with the following bound on failure probability.

\begin{lemma}[Lemma 2 of \citet{zanette2020learning}]\label{lem:zanette2}
	With probability at least $1-\delta/2$, for all $k\in [K]$, $h\in[H]$, $V_{h+1}\in\mathcal{V}_{h+1}$,
	\begin{equation}
	\left\|\sum_{i=1}^{k-1}\phi_h^i\left(r_h^i-r_h(s_h^i,a_h^i)+V_{h+1}(s_{h+1}^i)-\E_{s^\prime\sim P_h(\cdot|s_h^i,a_h^i)} V_{h+1}(s^\prime)\right)\right\|_{(\Sigma_h^k)^{-1}}\leq \sqrt{\beta_h^k},
	\end{equation}
	where $\sqrt{\beta_h^k}:=\sqrt{d_h \log(1+k/d_h)+2d_{h+1}\log(1+4\sqrt{k d_h})+\log(\frac{2KH}{\delta})}+1=\widetilde{O}(\sqrt{d_h+d_{h+1}})$.
\end{lemma}

Next, we define the ``best'' feasible parameters $\theta^\star$ that well approximate the $Q^\star$ values, and such parameters are going to be a feasible solution for the optimization problem (Definition 3.1). Then we state the accuracy bound of $\theta^\star$.

\begin{definition}[Best feasible approximant, Definition 4 of \citet{zanette2020learning}] 
	We recursively define the best approximant parameter $\theta^\star_h$ for $h\in[H]$ as:
	\begin{equation}
	\theta_h^\star=\arg\min_{\theta\in\mathcal{B}_h}\sup_{(s,a)}\left|\phi_h(s,a)^\top\theta-(\mathcal{T}_h Q_{h+1}(\theta^\star_{h+1}))(s,a)\right|
	\end{equation}
	with ties broken arbitrarily and $\theta^\star_{H+1}=0$.
\end{definition}

\begin{lemma}[Accuracy Bound of $\theta^\star$, Lemma 6 of \citet{zanette2020learning}]
	It holds that for all $h\in[H]$:
	\begin{equation}
	\sup_{(s,a)}|Q^\star_h(s,a)-\phi_h(s,a)^\top \theta_h^\star|\leq (H-h+1)\mathcal{I}.
	\end{equation}
\end{lemma}

For notational simplicity, for $\bar{\theta}_h$ which is the solution of Definition 3.1, we denote $\bar{Q}_{h}(\cdot,\cdot)=Q_h(\bar{\theta}_h)(\cdot,\cdot)=\phi_h(\cdot,\cdot)^\top\bar{\theta}_h$. Besides, $\bar{Q}_h^k$ represents $Q_h(\bar{\theta}_h^k)$ where $\bar{\theta}_h^k$ is the solution at the $k$-th episode. Similarly, $\bar{V}_h(\cdot)=V_h(\bar{\theta}_h)(\cdot)$ and $\bar{V}_h^k(\cdot)=V_h(\bar{\theta}_h^k)(\cdot)$. In addition, let $b_k$ denote the last policy update before episode $k$, for all $k\in[K]$.

\begin{lemma}[Optimism, Lemma 7 of \citet{zanette2020learning}]\label{lem:op}
	Under the high probability case in Lemma \ref{lem:zanette2}, if we choose $\sqrt{\alpha_h^k}=\sqrt{\beta_h^k}+\sqrt{k}\mathcal{I}+\sqrt{d_h}=\sqrt{k}\mathcal{I}+\widetilde{O}(\sqrt{d_h+d_{h+1}})$, then $\bar{\theta}_h=\theta_h^\star$, for all $h\in[H]$ is a feasible solution of the optimization problem (Definition 3.1). Therefore, for all $k\in[K]$, the optimistic value function satisfies
	\begin{equation}
	\bar{V}_1^{b_k}(s_1)\geq V_1^\star(s_1)-H\mathcal{I}.
	\end{equation}
\end{lemma}

In addition to optimism, we also have the following upper bound of Bellman error.

\begin{lemma}[Bound of Bellman error, Lemma 1 of \citet{zanette2020learning}]\label{lem:boundbellman}
	Under the high probability case in Lemma \ref{lem:zanette2}, it holds that for all $(k,h,s,a)\in[K]\times[H]\times\mathcal{S}\times\mathcal{A}$,
	\begin{equation}
	\begin{split}
	\left|\left(\bar{Q}_h^{b_k}-\mathcal{T}_h \bar{Q}_{h+1}^{b_k}\right)(s,a)\right|&\leq \mathcal{I}+\left\|\phi_h(s,a)\right\|_{(\Sigma_h^{b_k})^{-1}}\left(\sqrt{b_k}\mathcal{I}+\sqrt{\beta_h^{b_k}}+\sqrt{d_h}+\sqrt{\alpha_h^{b_k}}\right)\\&=\mathcal{I}+2\left\|\phi_h(s,a)\right\|_{(\Sigma_h^{b_k})^{-1}}\sqrt{\alpha_h^{b_k}}.
	\end{split}
	\end{equation}
\end{lemma}

Now we are ready to present the regret analysis of Algorithm \ref{alg:main}. 
\begin{proof}[Proof of regret bound]
	We prove based on the high probability case in Lemma \ref{lem:zanette2}.
	
	First of all, the regret over $K$ episodes can be decomposed as
	\begin{equation}\label{equ:decompose}
	\begin{split}
	&\text{Regret}(K) = \sum_{k=1}^{K}\left(V_{1}^\star(s_{1})-V_{1}^{\pi_{k}}(s_{1})\right)\\
	=& \sum_{k=1}^{K}\left(V_{1}^\star(s_{1})-V_{1}^{\pi_{b_k}}(s_{1})\right)\\
	=& \sum_{k=1}^{K}\left(V_{1}^\star(s_{1})-\bar{V}_{1}^{b_{k}}(s_{1})\right)+\sum_{k=1}^{K}\left(\bar{V}_{1}^{b_{k}}(s_1)-V_{1}^{\pi_{b_k}}(s_{1})\right)\\
	\leq& HK\mathcal{I}+\sum_{k=1}^{K}\left(\bar{V}_{1}^{b_{k}}(s_1)-V_{1}^{\pi_{b_k}}(s_{1})\right),
	\end{split}
	\end{equation}
	where the last inequality results from Lemma \ref{lem:op}.
	
	Note that $\bar{V}_h^{b_k}(s_h^k)=\bar{Q}_h^{b_k}(s_h^k,a_h^k)$ due to our choice of $\pi_k$, it holds that for all $k,h\in[K]\times[H]$,
	\begin{equation}\label{equ:single}
	\begin{split}
	&\left(\bar{V}_h^{b_k}-V_h^{\pi_{b_k}}\right)(s_h^k)=\bar{Q}_h^{b_k}(s_h^k,a_h^k)-\mathcal{T}_h\bar{Q}_{h+1}^{b_k}(s_h^k,a_h^k)+\mathcal{T}_h\bar{Q}_{h+1}^{b_k}(s_h^k,a_h^k)-V_h^{\pi_{b_k}}(s_h^k)\\\leq& \mathcal{I}+2\left\|\phi_h(s_h^k,a_h^k)\right\|_{(\Sigma_h^{b_k})^{-1}}\sqrt{\alpha_h^{b_k}}+\E_{s^\prime\sim P_h(\cdot|s_h^k,a_h^k)}\left(\bar{V}_{h+1}^{b_k}-V_{h+1}^{\pi_{b_k}}\right)(s^\prime),
	\end{split}
	\end{equation}
	where the inequality holds because of Lemma \ref{lem:boundbellman}.
	
	Plugging \eqref{equ:single} into \eqref{equ:decompose}, we have with probability $1-\delta$,
	\begin{equation}
	\begin{split}
	&\text{Regret}(K)\leq HK\mathcal{I}+\sum_{k=1}^{K}\left(\bar{V}_{1}^{b_{k}}(s_1)-V_{1}^{\pi_{b_k}}(s_{1})\right)\\\leq&HK\mathcal{I}+\sum_{k=1}^{K}\sum_{h=1}^{H}\left(\mathcal{I}+2\left\|\phi_h(s_h^k,a_h^k)\right\|_{(\Sigma_h^{b_k})^{-1}}\sqrt{\alpha_h^{b_k}}\right)\\&+\sum_{h=1}^H \sum_{k=1}^K \left(\E_{s^\prime\sim P_h(\cdot|s_h^k,a_h^k)}\left(\bar{V}_{h+1}^{b_k}-V_{h+1}^{\pi_{b_k}}\right)(s^\prime)-\left(\bar{V}_{h+1}^{b_k}-V_{h+1}^{\pi_{b_k}}\right)(s_{h+1}^{k})\right)\\\leq&
	HK\mathcal{I}+\sum_{k=1}^{K}\sum_{h=1}^{H}\left(\mathcal{I}+2\left\|\phi_h(s_h^k,a_h^k)\right\|_{(\Sigma_h^{b_k})^{-1}}\sqrt{\alpha_h^{b_k}}\right)+\widetilde{O}(\sum_{h=1}^H\sqrt{d_h K})\\=&
	2HK\mathcal{I}+\widetilde{O}(\sum_{h=1}^H\sqrt{d_h K})+\sum_{k=1}^{K}\sum_{h=1}^{H}2\left\|\phi_h(s_h^k,a_h^k)\right\|_{(\Sigma_h^{b_k})^{-1}}\sqrt{\alpha_h^{b_k}},
	\end{split}
	\end{equation}
	where the second inequality is because of \eqref{equ:single}. The last inequality holds with high probability due to Azuma-Hoeffding inequality (Lemma \ref{lem:azuma}) and the fact that $\|\bar{V}_{h+1}^{b_k}\|_\infty\leq \|\bar{\theta}_{h+1}^{b_k}\|_2\leq \sqrt{d_{h+1}}$ for any $k\in[K]$.
	
	Finally, it holds that
	\begin{equation}
	\begin{split}
	&\text{Regret}(K)\leq 2HK\mathcal{I}+\widetilde{O}(\sum_{h=1}^H\sqrt{d_h K})+\sum_{k=1}^{K}\sum_{h=1}^{H}2\left\|\phi_h(s_h^k,a_h^k)\right\|_{(\Sigma_h^{b_k})^{-1}}\sqrt{\alpha_h^{b_k}}\\\leq&
	2HK\mathcal{I}+\widetilde{O}(\sum_{h=1}^H\sqrt{d_h K})+2\sum_{h=1}^H\sqrt{\alpha_h^K}\cdot\sqrt{K\sum_{k=1}^K \left\|\phi_h(s_h^k,a_h^k)\right\|_{(\Sigma_h^{b_k})^{-1}}^2}\\\leq&
	\widetilde{O}(HK\mathcal{I}+\sum_{h=1}^H\sqrt{d_h K})+2\sum_{h=1}^H\sqrt{\alpha_h^K}\cdot\sqrt{K\sum_{k=1}^K 2\left\|\phi_h(s_h^k,a_h^k)\right\|_{(\Sigma_h^{k})^{-1}}^2}\\\leq&
	\widetilde{O}(HK\mathcal{I}+\sum_{h=1}^H\sqrt{d_h K})+2\sum_{h=1}^H\sqrt{\alpha_h^K}\cdot\sqrt{2K\cdot 2d_h\log(1+K)}\\\leq&
	\widetilde{O}(HK\mathcal{I}+\sum_{h=1}^H\sqrt{d_h K})+\sqrt{K\log(1+K)}\cdot\widetilde{O}\left(\sum_{h=1}^H \sqrt{Kd_h}\mathcal{I}+\sqrt{d_h(d_h+d_{h+1})}\right)\\\leq&
	\widetilde{O}(\sum_{h=1}^H \sqrt{d_h}K\mathcal{I}+\sum_{h=1}^H d_h\sqrt{K}),
	\end{split}
	\end{equation}
	where the second inequality holds according to Cauchy-Schwarz inequality and the fact that $\alpha_h^k$ is non-decreasing in $k$. The third inequality results from Lemma \ref{lem:double} and the fact that $\det((\Sigma_h^{b_k})^{-1})=\det(\Sigma_h^{b_k})^{-1}\leq 2\det(\Sigma_h^k)^{-1}=2\det((\Sigma_h^k)^{-1})$. The forth inequality is because of elliptical potential lemma (Lemma \ref{lem:potential}). The fifth inequality is derived by the definition of $\alpha_h^K$ (from Lemma \ref{lem:op}). The last inequality comes from direct calculation.
	
	The regret analysis is complete.
\end{proof}

\section{Proof of Theorem \ref{thm:lower}}
In this section, we prove our lower bound of switching cost.

\begin{theorem}[Restate Theorem \ref{thm:lower}]\label{thm:restatelower}
	Assume that the inherent Bellman error $\mathcal{I}=0$ and $d_h\geq 3$ for all $h\in[H]$, for any algorithm with sub-linear regret bound, the global switching cost is at least $\Omega(\sum_{h=1}^H d_h)$.
\end{theorem}

We first briefly discuss about our assumptions. We assume zero inherent Bellman error (\emph{i.e.} $\mathcal{I}=0$) since it is possible to derive sub-linear regret bounds only if $\mathcal{I}=0$, and we want to derive lower bounds of switching cost for algorithms with sub-linear regret. Otherwise, the regret bound will always be linear in $K$. Also, the assumption on $d_h\geq 3$ for all $h\in[H]$ is without loss of generality. 

\begin{proof}[Proof of Theorem \ref{thm:restatelower}]
	We first construct an MDP with two states, the initial state $s_1$ and the absorbing state $s_2$. 
	
	For absorbing state $s_2$, the choice of action is only $a_0$, while for initial state $s_1$, the choice of actions at layer $h$ is $\{a_1,a_2,\cdots,a_{d_h-1}\}$. Then we define the $d_h$-dimensional feature map for the $h$-th layer:
	$$\phi_h(s_2,a_0)=(1,0,0,\cdots,0),\;\;\phi_h(s_1,a_i)=(0,\cdots,0,1,0,\cdots),$$
	where for $s_1,a_i$ ($i\in[d_h-1]$), the $(i+1)$-th element is $1$ while all other elements are $0$. 
	
	We now define the transition kernel and reward function as $P_h(s_2|s_2,a_0)=1$, $r_h(s_2,a_0)=0$, $P_h(s_1|s_1,a_1)=1$, $r_h(s_1,a_1)=0$ for all $h\in[H]$. Besides, $P_h(s_2|s_1,a_i)=1$, $r_h(s_1,a_i)=r_{h,i}$ for all $h\in[H]$ and $2\leq i\leq d_h$, where $r_{h,i}$'s are unknown non-zero values. Note that such MDP has zero inherent Bellman error ($\mathcal{I}=0$) since the function class $\{\phi_h(s,a)^\top \theta_h\;|\;\theta_h\in\mathcal{B}_h\}$ includes all possible Q-value functions.
	
	Therefore, for any deterministic policy, the only possible case is that the agent takes action $a_1$ and stays at $s_1$ for the first $h-1$ steps, then at step $h$ the agent takes action $a_{i}$ ($i\geq2$) and transitions to $s_2$ with reward $r_{h,i}$, later the agent always stays at $s_2$ with no more reward. For this trajectory, the total reward will be $r_{h,i}$. Also, for any deterministic policy, the trajectory is fixed, like pulling an ``arm'' in multi-armed bandits setting. Note that the total number of such ``arms'' with non-zero unknown reward is at least $\sum_{h=1}^{H} (d_h-2)=\Omega(\sum_{h=1}^H d_h)$ due to our assumption that $d_h\geq 3$. Even if the transition kernel is known to the agent, this MDP is still as difficult as a multi-armed bandits problem with $\Omega(\sum_{h=1}^H d_h)$ arms. Together will Lemma \ref{lem:mab} below, the proof is complete.
\end{proof}

\begin{lemma}[Lemma H.4 of \citet{qiao2022sample}]\label{lem:mab}
	For any algorithm with sub-linear regret bound under $K$-armed bandit problem, the switching cost is at least $\Omega(K)$.
\end{lemma}

\section{Proof for Section \ref{sec:6}} \label{sec:C}

In this section, we prove the theorems regarding our Algorithm \ref{alg:minor} under the generalized linear function approximation setting. We begin with the upper bounds for switching cost and regret.

\subsection{Proof of upper bounds}

\begin{theorem}[Restate Theorem \ref{thm:minor}]\label{thm:upper2}
	The global switching cost of Algorithm \ref{alg:minor} is bounded by $O(dH\cdot \log K)$. In addition, with probability $1-\delta$, the regret of Algorithm \ref{alg:minor} over $K$ episodes is bounded by
	$$\text{Regret}(K)\leq \widetilde{O}\left(H\sqrt{d^3K}\right).$$
\end{theorem}

\begin{proof}[Proof of switching cost bound]
	Since the feature map in Algorithm \ref{alg:minor} satisfies that for all $s,a\in\mathcal{S}\times\mathcal{A}$, $\phi(s,a)\in\mathbb{B}_d=\{x\in\mathbb{R}^d:\|x\|_2\leq 1\}$, we have $\|\phi(s,a)\|_2\leq 1$. Therefore, the conclusion of Lemma \ref{lem:wang} still holds, with $d_h=d$ for all $h\in[H]$. In addition, because our policy update rule (line 7 of Algorithm \ref{alg:minor}) is identical to Algorithm \ref{alg:main}, the $O(dH\cdot\log K)$ upper bound of switching cost results from identical proof as in Section \ref{sec:ps}, with all $d_h$ replaced by $d$.
\end{proof}

Before we prove the upper bound of regret, we state some technical lemmas from \citet{wang2019optimism}.

\begin{lemma}[Corollary 3 of \citet{wang2019optimism}]\label{lem:decom}
	We denote the estimated Q value function of layer $h$ at the $k$-th episode by $Q_h^k(\cdot,\cdot)$. Suppose there exists a function $\text{conf}_h^k:\mathcal{S}\times\mathcal{A}\rightarrow \mathbb{R}^+$ such that for all $(k,h,s,a)\in[K]\times[H]\times\mathcal{S}\times\mathcal{A}$,
	\begin{equation}\label{equ:decom}
	Q_h^\star(s,a)\leq Q_h^k(s,a)\leq \mathcal{T}_h(Q_{h+1}^k)(s,a)+\text{conf}_h^k(s,a),
	\end{equation}
	(where $\mathcal{T}_h$ is Bellman operator) and the policy $\pi_k$ is the greedy policy with respect to $Q_h^k$, then with probability at least $1-\delta$,
	$$\text{Regret}(K)\leq \sum_{k=1}^K \sum_{h=1}^H \text{conf}_h^k(s_h^k,a_h^k)+O(H\sqrt{K\log (1/\delta)}).$$
\end{lemma}

Lemma \ref{lem:decom} is a standard regret decomposition which will be used to bound the regret of Algorithm \ref{alg:minor}. Below we give a valid choice of the confidence bound $\text{conf}_h^k$. Note that we define $b_k$ to be the last policy update before episode $k$, for all $k\in[K]$. Therefore, $Q_h^k=Q_h^{b_k}$ for all $k\in[K]$.

\begin{lemma}[Adapted from Lemma 6 of \citet{wang2019optimism}]\label{lem:con}
	With probability $1-\delta$, it holds that for all $k,h,s,a\in[K]\times[H]\times\mathcal{S}\times\mathcal{A}$,
	$$\left|f(\langle\phi(s,a),\theta_h^k\rangle)-\mathcal{T}_h(Q_{h+1}^k)(s,a)\right|\leq \gamma \|\phi(s,a)\|_{(\Sigma_h^{b_k})^{-1}},$$
	where $\gamma$ is defined in Algorithm \ref{alg:minor}. 
\end{lemma}

Therefore, optimism is straightforward.

\begin{lemma}[Corollary 5 of \citet{wang2019optimism}]\label{lem:o}
	Under the high probability case in Lemma \ref{lem:con}, for all $k,h,s,a\in[K]\times[H]\times\mathcal{S}\times\mathcal{A}$, $Q_h^k(s,a)\geq Q_h^\star(s,a)$.
\end{lemma}

Combining optimism (Lemma \ref{lem:o}) with Lemma \ref{lem:con}, we have that $Q_h^k$ in Algorithm \ref{alg:minor} satisfies condition \eqref{equ:decom} with $\text{conf}_h^k(s,a)=\gamma \|\phi(s,a)\|_{(\Sigma_h^{b_k})^{-1}}$. Below we bound the summation of bonus.

\begin{lemma}\label{lem:sumconf}
	Assume that $\text{conf}^k_h(s,a)=\gamma\|\phi(s,a)\|_{(\Sigma_h^{b_k})^{-1}}$, then it holds that
	\begin{equation}
	\sum_{k=1}^K \sum_{h=1}^H \text{conf}_h^k(s_h^k,a_h^k)=\sum_{k=1}^K \sum_{h=1}^H \gamma\|\phi(s_h^k,a_h^k)\|_{(\Sigma_h^{b_k})^{-1}}\leq H\gamma\sqrt{4Kd\log(1+K)}.
	\end{equation}
\end{lemma}

\begin{proof}[Proof of Lemma \ref{lem:sumconf}]
	\begin{equation}
	\begin{split}
	&\sum_{k=1}^K \sum_{h=1}^H \gamma\|\phi(s_h^k,a_h^k)\|_{(\Sigma_h^{b_k})^{-1}}\leq \sum_{h=1}^H\gamma\sqrt{K\sum_{k=1}^K \left\|\phi(s_h^k,a_h^k)\right\|_{(\Sigma_h^{b_k})^{-1}}^2}\\\leq&
	\sum_{h=1}^H\gamma\sqrt{K\sum_{k=1}^K 2\left\|\phi(s_h^k,a_h^k)\right\|_{(\Sigma_h^{k})^{-1}}^2}\\\leq&
	\sum_{h=1}^H\gamma\sqrt{2K\cdot 2d\log(1+K)}\\=&
	H\gamma\sqrt{4Kd\log(1+K)},
	\end{split}
	\end{equation}
	where the first inequality holds according to Cauchy-Schwarz inequality. The second inequality results from Lemma \ref{lem:double} and the fact that $\det((\Sigma_h^{b_k})^{-1})=\det(\Sigma_h^{b_k})^{-1}\leq 2\det(\Sigma_h^k)^{-1}=2\det((\Sigma_h^k)^{-1})$. The third inequality is because of elliptical potential lemma (Lemma \ref{lem:potential}). 
\end{proof}

Now we are ready to present the proof of the regret upper bound.

\begin{proof}[Proof of regret upper bound]
	The final $\widetilde{O}(H\sqrt{d^3 K})$ regret upper bound is derived by combining Lemma \ref{lem:decom}, Lemma \ref{lem:sumconf} and the definition that $\gamma=\widetilde{O}(d)$.
\end{proof}

\subsection{Proof of lower bound}
Finally, we present the proof of the lower bound.

\begin{theorem}[Restate Theorem \ref{thm:lower2}]\label{thm:relower2}
	For any algorithm with sub-linear regret bound, the global switching cost is at least $\Omega(dH)$.
\end{theorem}

\begin{proof}[Proof of Theorem \ref{thm:relower2}]
	Since linear MDP is a special case of generalized linear function approximation, the $\Omega(dH)$ lower bound of global switching cost in \citet{gao2021provably} holds here.
\end{proof}

\section{Assisting technical lemmas}

\begin{lemma}[Azuma-Hoeffding inequality]\label{lem:azuma}
	Let $X_i$ be a martingale difference sequence such that $X_i\in[-A,A]$ for some $A>0$. Then with probability at least $1-\delta$, it holds that:
	$$\left|\sum_{i=1}^n X_i\right|\leq \sqrt{2A^2 n\log(\frac{1}{\delta})}.$$
\end{lemma}

\begin{lemma}[Lemma C.1 of \citet{wang2021provably}]\label{lem:wang}
	Let $\{\Sigma_h^k\}_{(h,k)\in[H]\times[K]}$ be as defined in Algorithm \ref{alg:main}. Then for all $h \in [H]$ and $k \in [K]$, we have $\det(\Sigma_h^k)\leq (\lambda + \frac{k - 1}{d_h})^{d_h}$.
\end{lemma}

\begin{lemma}[Lemma 12 of \citet{abbasi2011improved}]\label{lem:double}
	Suppose $A, B \in \mathbb{R}^{d\times d}$ are two positive definite matrices satisfying that $A\succcurlyeq B$, then for any $x \in \mathbb{R}^d$, we have
	$$\frac{\|x\|_A^2}{\|x\|_B^2}\leq \frac{\det(A)}{\det(B)}.$$
\end{lemma}

\begin{lemma}[Elliptical Potential Lemma, Lemma 26 of \citet{agarwal2020flambe}]\label{lem:potential}
	Consider a sequence of $d\times d$ positive semi-definite matrices $X_1,\cdots,X_T$ with $\max_t Tr(X_t) \leq 1$ and define $M_0=I,\cdots,M_t=M_{t-1}+X_t$. Then 
	$$\sum_{t=1}^{T} Tr(X_t M_{t-1}^{-1})\leq 2d\log(1+\frac{T}{d}).$$
\end{lemma}

\end{document}